\documentclass[sigconf]{acmart}
\usepackage{amsmath,amsfonts}
\newtheorem{theorem}{Theorem}
\usepackage{multirow}
\usepackage{tabularx}

\AtBeginDocument{%
  \providecommand\BibTeX{{%
    \normalfont B\kern-0.5em{\scshape i\kern-0.25em b}\kern-0.8em\TeX}}}

\setcopyright{acmcopyright}
\copyrightyear{2018}
\acmYear{2018}
\acmDOI{XXXXXXX.XXXXXXX}

\acmConference[Conference acronym 'XX]{Make sure to enter the correct
  conference title from your rights confirmation emai}{June 03--05,
  2018}{Woodstock, NY}
\acmPrice{15.00}
\acmISBN{978-1-4503-XXXX-X/18/06}

\begin{document}

\title{GradXKG: A Universal Explain-per-use Temporal Knowledge Graph Explainer}


\author{Chenhan Yuan}
\affiliation{%
\institution{The University of Manchester}
   \city{Manchester}
   \country{UK}}
 \email{chenhan.yuan@manchester.ac.uk}

\author{Hoda Eldardiry}
\affiliation{%
   \institution{Virginia Tech}
   \city{Blacksburg}
   \country{U.S.}}
 \email{hdardiry@vt.edu}


\begin{abstract}
Temporal knowledge graphs (TKGs) have shown promise for reasoning tasks by incorporating a temporal dimension to represent how facts evolve over time. However, existing TKG reasoning (TKGR) models lack explainability due to their black-box nature. Recent work has attempted to address this through customized model architectures that generate reasoning paths, but these recent approaches have limited generalizability and provide sparse explanatory output. To enable interpretability for most TKGR models, we propose GradXKG, a novel two-stage gradient-based approach for explaining Relational Graph Convolution Network (RGCN)-based TKGR models. First, a Grad-CAM-inspired RGCN explainer tracks gradients to quantify each node's contribution across timesteps in an efficient "explain-per-use" fashion. Second, an integrated gradients explainer consolidates importance scores for RGCN outputs, extending compatibility across diverse TKGR architectures based on RGCN. Together, the two explainers highlight the most critical nodes at each timestep for a given prediction. Our extensive experiments demonstrated that, by leveraging gradient information, GradXKG provides insightful explanations grounded in the model's logic in a timely manner for most RGCN-based TKGR models. This helps address the lack of interpretability in existing TKGR models and provides a universal explanation approach applicable across various models. 
\end{abstract}




\received{20 February 2007}
\received[revised]{12 March 2009}
\received[accepted]{5 June 2009}

\maketitle

\section{Introduction}
Since knowledge graphs are dynamic in nature, i.e., evolve over time, Temporal Knowledge Graphs (TKG) have a promising potential in the fields of question answering~\cite{saxena2021question,jia2021complex}, event forecasting~\cite{deng2020dynamic,li2021temporal}, and information retrieval~\cite{gottschalk2018eventkg}. Unlike conventional static knowledge graphs, which represent each fact with a triplet $(subject, relation, object)$, temporal knowledge graphs incorporate a temporal dimension to represent how facts and relations evolve over time. In general, a temporal knowledge graph represents each fact using a quadruple $(subject, relation, object, timestamp)$. Conventionally, a temporal knowledge graph is represented by decomposing it into a sequence of static knowledge graphs, each of which contains all facts at the corresponding timestamp. 

TKGs provide new perspectives and insights for many
downstream applications, e.g., disease diagnosis aid\cite{diao2021research} and stock prediction~\cite{deng2019knowledge}. The unique promise of TKGs has sparked a growing interest in reasoning over TKG. TKG reasoning (TKGR) is a task to validate whether a query relationship between two entities is true, given the context provided by the TKG. With the rise of graph neural networks (GNN), most existing TKGR methods first leverage Relational Graph Convolutional Networks (RGCNs)~\cite{schlichtkrull2018modeling}, a type of GNN well-suited for multi-relational graphs like TKGs, to encode the local graph structure into dense vector representations. Various neural architectures are then applied on top of the RGCN encodings to score the validity of the query relationship. For instance, curriculum learning and neural ordinary differential equations have been used to enhance TKGR performance~\cite{han2021learning,li2022complex}. While RGCN-based TKGR models have achieved significant improvements over general neural network models, RGCN-based TKGR models still lack explainability due to their end-to-end black-box nature. The reasoning process behind their predictions is opaque. This lack of explainability is a critical limitation because explainability is crucial for trustworthy AI systems. 

Recent work has attempted to address this limitation by designing customized model architectures that can generate reasoning paths along with the prediction results. For example, TimeTraveler~\cite{sun2021timetraveler} proposed a reinforcement learning-based explainable TKGR model that works as follows. Given a query: \emph{(Governor (Cote d'Ivoire), Make an appeal or request, ?, 2018/10/14)}, it generates the reasoning path: \emph{(Governor (Cote d'Ivoire), Praise or endorse, Party Member, 2018/10/12)$\rightarrow$(Party Member, Make an appeal or request, Citizen, 2018/9/29)}. However, existing explainable TKGR models have two critical problems: 1) They are constrained by requiring custom-designed model architectures to enable reasoning path generation. This leads to a lack of generalizability, which is a major drawback, as it prevents easy application to the myriad of most non-explainable TKGR models, i.e., RGCN-based models. 2) As illustrated in Fig~\ref{fig:comp_xtkgr}, current explainable TKGR models only provide reasoning paths with very few nodes per static knowledge graph. This sparse explanatory output fails to capture the complex interdependencies within the broader knowledge graph context. In particular, it is difficult to fully trust or evaluate the faithfulness of explanations without the ability to account for the potential impact of other entities. 

To address these challenges, it is necessary to design a universal explainer that can work for most RGCN-based TKGR models. A simple and intuitive possible approach is to follow the GNNExplainer method, by masking a node/edge in the TKG, and observing its impact on the TKGR model prediction performance. ~\cite{ying2019gnnexplainer}. 
However, to obtain the importance of all nodes across all timestamps, the GNNExplainer-based method requires repeating the test at least $\mathcal{O}(n*m)$ times; where $n$ is the number of nodes and $m$ is the number of timestamps for various inputs; as shown in Fig~\ref{fig:comp_xgnn}. This is an exhaustive process that scales poorly as the knowledge graph grows. We propose a more efficient approach to design a universal explainer. More specifically, we propose a universal TKGR model explainer that works by analyzing the model's gradients. In doing so, it provides real-time rationales for each individual prediction, elucidating the model's logic in an accessible ``explain-per-use'' fashion.

\begin{figure}[]
\centering
\includegraphics[width=0.95\columnwidth]{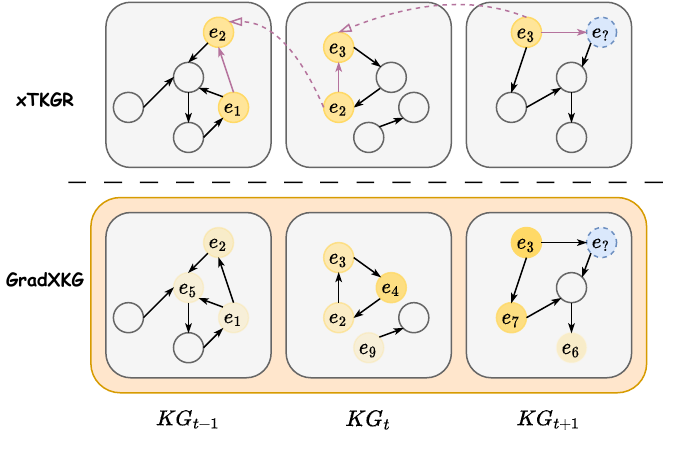}
\caption{Comparison between the proposed GradXKG and conventional explainable TKGR models (xTKGR). Node color shade represents node importance; darker yellow is more important than lighter yellow. Dashed arrows denote the same entities in different timestamps. $e?$ is the query entity.}
\label{fig:comp_xtkgr}
\end{figure}

\begin{figure}[]
\centering
\includegraphics[width=0.95\columnwidth]{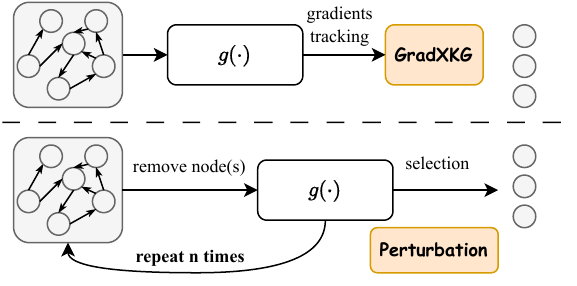}
\caption{Comparison between the proposed GradXKG and perturbation-based methods. $g(\cdot)$ denotes TKGR models. Perturbation-based methods require at least $n$ times running to determine node importance in TKG.}
\label{fig:comp_xgnn}
\end{figure}

In this paper, we propose GradXKG: a novel two-stage approach for generating explanations of RGCN-based temporal knowledge graph reasoning (TKGR) models. Our proposed GradXKG approach bridges a critical gap in universal explainable temporal knowledge graph reasoning. Our method leverages gradient information to highlight the most influential nodes across timesteps. First, a Grad-CAM-inspired RGCN explainer tracks gradients flowing into each node at each timestep, quantifying the contribution of that node to the RGCN output tensors in each run. Secondly, to extend the compatibility of our proposed approach to TKGR architectures that utilize RGCN, we propose an integrated gradients-based top-layer explainer. This proposed explainer layer tracks the contribution of RGCN output tensors toward the final prediction in the top layer of these TKGR models. Moreover, the top-layer explainer associates RGCN output tensors with prediction scores regardless of the model specifics. This allows our approach to be applicable to any neural network-based RGCN TKGR model. Together, the two explainers produce a heat map spotlighting the most critical nodes at each timestep for a given prediction. Going beyond perturbation-based approaches, our proposed ``explain-per-use'' mechanism provides explanations that are tailored to each individual prediction in a time-efficient manner. Furthermore, GradXKG enables interpretability within most RGCN-based TKGR models. This is achieved by generating explanations that are informed by gradient information.

Our main contributions can be summarized as follows:
\begin{itemize}
    \item We proposed the first novel gradient-based two-stage RGCN-based TKGR explanation approach (GradXKG) that can explain all RGCN-based TKGR models. To the best of our knowledge, this is the first effort in this research area.
    \item We proposed a Grad-CAM-based RGCN-explainer in an ``explain-per-use'' fashion, which significantly reduces time complexity compared to perturbation methods. 
    \item We proved that our proposed explainer can be approximated to the Grad-CAM GCN model if both are unsigned.
    \item Our extensive experiments demonstrate that the proposed GradXKG can outperform other explanation methods in terms of explanation sufficiency and accuracy.
\end{itemize}
\section{Related Work}
\subsection{Temporal Knowledge Graph Reasoning Models}
Temporal knowledge graph reasoning is an emerging field that aims to model the evolution of real-world events and their relationships over time. This task involves two key settings: interpolation, which focuses on completing knowledge graphs over a given time span, and extrapolation, which forecasts future facts based on historical data~\cite{li2021temporal,li2022complex}. Early work on extrapolation like Know-Evolve utilized temporal point processes to capture continuous-time dynamics~\cite{trivedi2017know}. However, graph neural networks (GNNs) have since become the dominant modeling approach due to their ability to encode both structural and temporal dependencies. In particular, relational GCNs (RGCNs) have emerged as a powerful tool for temporal knowledge graph reasoning~\cite{schlichtkrull2018modeling}. RE-NET pioneered the use of RGCNs with an autoregressive encoder to model long-term temporal patterns~\cite{jin2020recurrent}. RE-GCN further improved modeling by emphasizing graph dependency learning~\cite{li2021temporal}. Later, Zhang et al. proposed a hierarchical RGCN to learn both long global and short local representations of temporal knowledge graphs~\cite{zhang2020relational}. Han et al. adapted neural ordinary differential equations on RGCN so that they can be applied in continuous space as temporal knowledge graphs vary continuously over time~\cite{han2021learning}. To better forecast unseen entities and relations, MTKGE utilized meta-learning to sample from existing temporal knowledge graphs to simulate future scenarios~\cite{xia2022metatkg}. CENET learned both historical and non-historical dependency and distinguished the dependency type of predicted events using contrastive learning~\cite{xu2023temporal}. Li et al.consider the length, diversity, and time-variability of evolutional patterns by introducing curriculum learning and online learning~\cite{li2022complex}. Overall, the ubiquity of RGCNs in TGKR models demonstrates their importance for encoding local graph structures, on top of which more complex temporal reasoning techniques can be developed.

Several other approaches generate predictions along with validated reasons. These explainable techniques can be categorized into three main types: logic rule-based, reinforcement learning-based, and attention-based. Logic rule mining is a popular technique for explainable forecasting on temporal knowledge graphs. TLogic extracts logic rules from the graph via temporal walk estimation and then applies the rules to make predictions~\cite{liu2022tlogic}. Lin et al. train encoders to incorporate both graph structures and logic rules\cite{lin2023techs}. Reinforcement learning has also been leveraged for explainable reasoning. Sun et al. developed an RL agent that travels optimal paths on the knowledge graph to predict future events~\cite{sun2021timetraveler}. The agent's trajectory explains its decision-making. Similarly, Li et al. first find relevant event clusters and then use RL search to forecast those subgroups~\cite{li2021search}. Attention mechanisms are another way to enable explainability. Some methods learn to expand an initial query by attending to important neighboring nodes~\cite{han2021explainable}. The attended subgraphs indicate the influential regions for prediction. Jung et al. also apply graph attention to iteratively propagate attention weights towards target nodes and use these weights to provide model interpretability~\cite{jung2021learning}.
\subsection{Graph Neural Networks Explanation}
GNN explainers can be broadly classified into two distinct categories: perturbation-based methods and gradient-based methods~\cite{yuan2022explainability}. Perturbation-based explainers identify important graph components (nodes, edges, or features) by removing or masking components and measuring the resulting impact on model predictions. A significant change in predictions when a component is masked indicates high relevance. For example, GNNExplainer identifies compact subgraph structures that are important for a prediction by using mutual information to quantify the difference between predictions on the original versus perturbed graph~\cite{ying2019gnnexplainer}. SubgraphX also masks substructures but uses Shapley values to measure subgraph importance based on contribution to the model output~\cite{yuan2021explainability}. Another perturbation-based method called PGM-Explainer illustrates dependencies between important features and provides interpretations of a model's reasoning using Bayesian network concepts~\cite{vu2020pgm}.

In contrast, gradient-based explainers focus directly on analyzing gradient information flow through the neural network model itself. The key idea is that components with high gradient magnitude likely have high relevance or impact on predictions. For example, GNN-LRP utilizes layer-wise relevance propagation, a technique that redistributes relevance scores based on neural network activations, to generate detailed explanations of predictions~\cite{bach2015pixel,schnake2021higher}. GraphLIME adapts a popular local explanation method called LIME, which stands for Local Interpretable Model-Agnostic Explanations, to graph domains using nonlinear feature selection~\cite{ribeiro2016should,huang2022graphlime}. Some methods extend gradient visualization techniques like Grad-CAM, originally designed for convolutional neural networks, to graph neural networks to improve explainability~\cite{pope2019explainability,selvaraju2017grad,oquab2015object}.
\begin{figure}[]
\centering
\includegraphics[width=0.95\columnwidth]{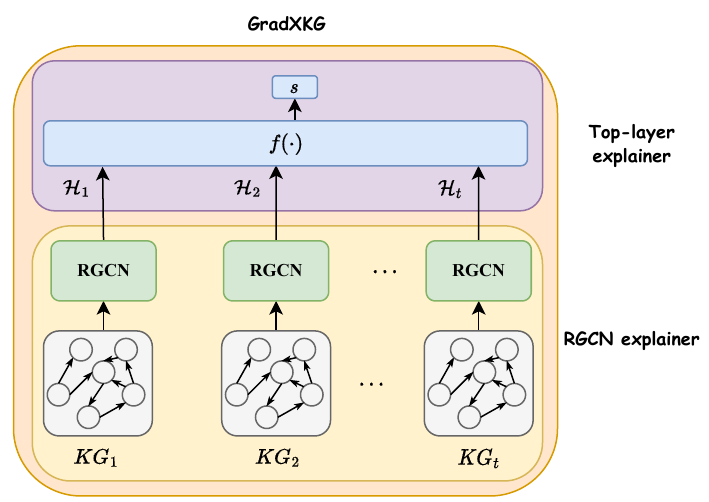}
\caption{General architecture of an RGCN-based TKGR model. We propose to explain an RGCN-based TKGR model by decomposing it into two sub-modules: RGCN explainer and Top-layer explainer.}
\label{fig:main}
\end{figure}
\section{Proposed GradXKG Method}
Conventionally, as shown in Fig.~\ref{fig:main}, an RGCN-based temporal knowledge graph reasoning model utilizes the Relational Graph Convolutional Network (RGCN) as the base model to encode graph information in each individual timestamp. On top of RGCN outputs, various reasoning methods are deployed, such as curriculum learning~\cite{li2022complex}, and autoregressive encoder~\cite{jin2020recurrent}, to yield a score ranging from 0-1 that indicates if the query triplet will happen in the future or not. The process can be formally defined as follows: Suppose a set $\mathcal{G}=\{\mathcal{G}_1, \mathcal{G}_2, \cdots. \mathcal{G}_T\}$ represents the individual static knowledge graph in each timestamp $t$, a regular temporal knowledge graph reasoning model can be defined as:
\begin{equation}
\label{eq:tkgr}
    \begin{aligned}
    s &= f(\mathcal{H}_1,\mathcal{H}_2,\cdots,\mathcal{H}_t)\\
    \mathcal{H}_i &= RGCN(\mathcal{G}_i) 
    \end{aligned}
\end{equation} where $RGCN(\cdot)$ is the RGCN model that takes each static knowledge graph as input. $f(\cdot)$ represents the complex temporal knowledge graph reasoning model that outputs a scalar score $s$. As shown in the Eq.~\ref{eq:tkgr}, our explanation method can also be decomposed into two parts: Grad-CAM-based RGCN explanation, i.e., $RGCN(\cdot)$ explainer, and a universal top-layer explanation, i.e., $f(\cdot)$ explainer.
\subsection{Gradient-based RGCN Explainer}
We start deriving our method from a simple setting: Assume that only the RGCN model itself is used in the temporal knowledge graph reasoning model. Then we have $s=RGCN(\mathcal{G}_i)$ following Eq.~\ref{eq:tkgr}. That is, the output of RGCN is a scalar that indicates the prediction of the whole model. We propose to design a gradient-based RGCN explainer method to explain this simple architecture.

Currently, the most popular gradient-based convolutional neural network (CNN) explanation method, Grad-CAM, has been successfully extended to graph convolutional networks (GCN) as GCN shares the similar convolutional layer as CNN~\cite{pope2019explainability,selvaraju2017grad}. The key idea of GCN Grad-CAM is to use the gradients of the target classification flowing into the final convolutional layer to produce a coarse localization map highlighting the important nodes for predicting the classification result. Formally, the first step of GCN-based Grad-CAM is to determine the $k$’th graph convolutional feature map at layer $l$:
\begin{equation}
\label{eq:def_gcn}
M^l_k(X,A) = \sigma(\tilde{D}^{-\frac{1}{2}}\tilde{A}\tilde{D}^{-\frac{1}{2}}M^{l-1}(X,A)W_k^l) 
\end{equation}where $X\in \mathbf{R}^{N\times d_x}$ is the representation of the nodes, $N$ is the number of nodes in the graph, and $A\in \mathbf{R}^{d_x\times d_x}$is the adjacency matrix. $\tilde{A} = A + I_{N}$ is the adjacency matrix of the undirected graph with added self-connections. $I_{N}$ is the identity matrix, $\tilde{D}_{ii} = \sum_j \tilde{A}_{ij}$, and $W_k^l$ is the $k$-th column of matrix $W_k$. Then the importance score $I_n$ can be used to generate a heat map for a node $n$ defined as follows:
\begin{equation}
    \begin{aligned}
    \alpha_k^l &= \frac{1}{N}\sum_{n=1}^{N}\frac{\partial y}{\partial M^l_{k,n}(X,A)}\\
    I_n^{GCN} &= \frac{1}{L}\sum_{l=1}^{L}ReLU(\sum_{k}\alpha_k^lM^l_{k,n}(X,A))\\
    \end{aligned}
\label{eq:gcn_score}
\end{equation}where $M^l_{k,n}(X,A))$ denotes the $k$’th feature at the $l$’th layer of node $n$. That is, the importance score of each node is a weighted sum over all feature columns associated with each node. However, despite the success of this simple yet effective gradient-based approach, it cannot be directly applied to the RGCN model. This is because the RGCN model binds different relational edges with different matrices so no $W_k^l$ can be defined for the RGCN model. This can be formally observed in the following definition of the RGCN model:
\begin{equation}
\label{eq:def_rgcn}
h_i^{(l+1)}= \sigma \left( \sum_{r \in \mathcal{R}}\sum_{j \in \mathcal{N}^r_i} \frac{1}{c_{i,r}}W_r^{(l)} h_j^{(l)} + W_0^{(l)}h_i^{(l)} \right),
\end{equation}
where $\mathcal{N}^r_i$ denotes the set of neighbor indices of node $i$ under relation $r\in\mathcal{R}$ and $h_i^(l)$ is the hidden state of node $i$ in layer $l$. $c_{i,r}$ is a normalization constant. As shown in the Eq.~\ref{eq:def_rgcn}, the hidden representation of node $i$ in layer $l+1$ is contributed by a set of relation matrices $W_r^{l}$ instead of one $W^l$. Therefore, Eq.~\ref{eq:gcn_score} cannot be directly applied to the RGCN model.

To extend the compatibility of GCN Grad-CAM to RGCN, we first convert Eq.~\ref{eq:def_rgcn} into matrix graph representation form:
\begin{equation}
\label{eq:def_rgcn_graph}
H^{(l+1)}= \sigma \left( \sum_{r\in R} A_rH^{(l)}W_r^{(l)} + I_{N}H^{(l)}W_0^{(l)}\right)
\end{equation}where $A_r$ is a relation-specific adjacency matrix as shown in the following equation. 
\begin{equation}
A_{r,i,j}=\left\{
\begin{array}{rcl}
1 & & {E(i,j)\in r}\\
0 & & {else}\\
\end{array} \right.
\end{equation} where $E(i,j)\in \mathcal{r}$ means the edge between nodes $i$ and $j$ represents relation type $r$. In this way, we can define the $k$-th RGCN feature map at layer $l$ under each relation individually. Formally, combining Eq.~\ref{eq:def_rgcn_graph} and Eq.~\ref{eq:gcn_score}, the $k$-th RGCN feature map under relation $r$ can be derived as follows: 
\begin{equation}
H^{(l+1)}_{k,r}= \sigma \left( A_rH^{(l)}W_{r,k}^{(l)} + I_{N}H^{(l)}W_{0,k}^{(l)}\right)
\end{equation}
Then the relation-dependent gradient-based weight $\alpha_{k,r}^l$ for $k$-th feature in $l$-th layer under $r$-th relation can be derived as follows:
\begin{equation}
\label{eq:rgcn_weight}
\begin{aligned}
    \alpha_{k,r}^l &= \frac{1}{N}\sum_{n=1}^{N}\frac{\partial s}{\partial H^l_{k,n,r}(X,A)}\\
    &=\frac{1}{N}\sum_{n=1}^{N}\frac{\partial y}{\partial \sigma(\cdot)}\frac{\partial \sigma(\cdot)}{\partial H^{l-1}_{k,n,r}(X,A)}\left( A_rW_{r,k}^{(l))}+I_NW_{0,k}^{(l)}\right)\\
    \end{aligned}
\end{equation} where $s$ is the final score. In this way, the weights derived here are always associated with one specific relation.

Following a similar procedure as in Eq.~\ref{eq:gcn_score}, the importance score $I_{n,r,l}$ for node $n$ from layer $l$ under relation $r$ is defined as follows:
\begin{equation}
    I_{n,r,l}^{RGCN} = ReLU(\sum_{k}\alpha_{k,r}^lH^l_{k,n,r}(X,A))
\label{eq:rgcn_final}
\end{equation} We use the average $I_{n,r,l}$ as the final importance score:
\begin{equation}
\label{eq:score_}
    I_{n,r}^{RGCN} = \frac{1}{L}\sum_{l=1}^{L}I_{n,r,l}^{RGCN}
\end{equation} Note that the implementation complexity of Eq.~\ref{eq:rgcn_final} increases as the number of relations increases. Therefore, we propose the following approximation that averages over the relation $r$ dimension:
\begin{equation}
\label{eq:score}
    I_{n}^{RGCN} \approx \frac{1}{R} \sum_{r=1}^{R} \frac{1}{L}\sum_{l=1}^{L}I_{n,r,l}^{RGCN}
\end{equation}
\subsection{Theoretical Analysis of RGCN Grad-CAM} We provide a theoretical analysis of the proposed Grad-CAM-based RGCN explainer in this section. First, note that the derived RGCN explainer is equivalent to the original GCN Grad-CAM if the output is unsigned. Formally, we define the equivalence theorem as follows: 
\begin{theorem}
$I_{n}^{RGCN}\approx I_{n}^{GCN}$\; iff. both $I_{n}^{RGCN}$ and $I_{n}^{GCN}$ are unsigned. 
\end{theorem}
\begin{proof}
To prove the equivalence theorem, we first convert Eq.~\ref{eq:gcn_score} into unsigned format by removing the $ReLU(\cdot)$ function: 
\begin{equation}
    I_n^{GCN} = \frac{1}{L}\sum_{l=1}^{L}\sum_{k}\alpha_k^lM^l_{k,n}(X,A)
\end{equation}
Then the $I_n^{RGCN}$ can be derived as follows:
\begin{equation}
    \begin{aligned}
    I_n^{RGCN}&=\frac{1}{R}\sum_{r=1}^{R}\frac{1}{L}\sum_{l=1}^{L}I_{n,r,l}\\
    &=\frac{1}{R}\sum_{r=1}^{R}\frac{1}{L}\sum_{l=1}^{L} ReLU(\sum_{k}\alpha_{k,r}^lH^l_{k,n,r}(X,A))\\
    &\Rightarrow \frac{1}{L}\sum_{l=1}^{L}\sum_{k}\frac{1}{R}\sum_{r=1}^{R}\alpha_{k,r}^lH^l_{k,n,r}(X,A)\;\text{(unsigned)}\\
    &= \frac{1}{L}\sum_{l=1}^{L}\sum_{k}\beta_{k,n}^{l}(X,A)
    \end{aligned}
\end{equation} Comparing $I_n^{RGCN}$ with $I_n^{GCN}$, the only difference is between $\beta_{k,n}^{l}(X,A)$ and $\alpha_k^lM^l_{k,n}(X,A)$. However, we consider that $\beta_{k,n}^{l}(X,A)$ represents the averaged $k$-th feature along $r$ dimension, which is equivalent $\alpha_k^lM^l_{k,n}(X,A)$. Therefore, we have $I_{n}^{RGCN}\approx I_{n}^{GCN}$. 
\end{proof}
Directly applying Eq.~\ref{eq:def_rgcn} to relational graphs in practice can lead to over-fitting on rare relations as the number of relations can rapidly grow. Therefore, the original RGCN provides a linear combination of basis transformations to represent relation matrices $W_r$. It is defined as follows:
\begin{equation}
    W_r^{(l)} = \sum_{b=1}^{B}a_{rb}^{(l)}V_b^{(l)}
\end{equation} where $V_b^{l}\in \mathbf{R}^{d_r\times d_o}$ is the $b$-th basis in layer $l$. Note that this decomposition method is consistent with Theorem 1. 
\begin{theorem}
THEOREM 1 holds when $W_r^{(l)} = \sum_{b=1}^{B}a_{rb}^{(l)}V_b^{(l)}$
\end{theorem}
\begin{proof}
To prove Theorem 2, we first convert the decomposition equation into a matrix form:
\begin{equation}
    W_r^{(l)} = U_r^{l}B^{(l)}
\end{equation} where $B^{l}\in \mathbf{R}^{B\times d_r \times d_o}$ is the stack of the set of all basis. $U_r^{l}\in \mathbf{R}^{1\times B}$ is a linear combination of all coefficients $a_{rb}^{(l)}$. We can therefore derive the hidden layer of RGCN under basis decomposition as follows:
\begin{equation}
    H^{(l+1)}= \sigma \left( \sum_{r\in R} A_rH^{(l)}[U_r^{l}B^{(l)}] + I_{N}H^{(l)}W_0^{(l)}\right)
\end{equation} Note that $[U_r^{l}B^{(l)}]$ is equivalent to the original $W_r^{l}$ in terms of matrix shape. Therefore, the computation of $\alpha_{k,r}^{l}$ and $H_{k,n,r}^{l}(X,A)$ (Eq.~\ref{eq:rgcn_weight}$\sim$\ref{eq:score})remains unchanged. As the final $I_n^{RGCN}$ preserves the same format, Theorem 1 still holds.
\end{proof}


\subsection{Top-Layer Integration}
\subsubsection{Top-Layer Explanation Method}
We proposed a top-layer explanation method that is model-agnostic so it can work with any model type in the layers above the RGCN. In particular, we utilize the integrated gradients method to quantify the contribution of each tensor to the model prediction for the layers above the RGCN. Integrated gradients is an explanation method by attributing a prediction to individual features~\cite{sundararajan2017axiomatic}. The integrated gradients method then calculates the contribution of each feature by integrating the gradients of the model's output with respect to that feature. This calculation is carried out along a path from a baseline input to the actual input. Formally, the integrated gradients for a one-dimensional tensor $x$ can be defined as follows:
\begin{equation}
\label{eq:ig}
IG(x)::=(x-x')\int_{\alpha=0}^{1}\frac{\partial f(x'+\alpha(x-x'))}{\partial x}\,d\alpha
\end{equation}where $f(\cdot)$ denotes the general neural network function and $x'$ is the base input. In practice, $x'$ is set to random or zeros.

Note that the integral part of Eq.~\ref{eq:ig} is not applicable in neural networks. Therefore, we approximate the Eq.~\ref{eq:ig} by summing over small intervals. Then we have:
\begin{equation}
\label{eq:ig_}
IG(x)\approx (x-x')\sum_{\alpha=1}^{p}\frac{\partial f(x'+\frac{p}{\alpha}(x-x'))}{\partial x}\frac{1}{p}
\end{equation}where $p$ is the sampling times.
\subsubsection{Explanation Integration}
Our proposed RGCN explainer provides node-level explanations by quantifying each node's contribution to the final predicted scalar score $s$, as shown in Eq.~\ref{eq:rgcn_weight}. However, this formulation assumes that the RGCN model outputs $s$ directly, which is not always true, since many models integrate RGCN as an intermediate component. In this case, additional processing is applied to the RGCN outputs before producing the final prediction. Our proposed RGCN explainer further considers cases when models include intermediate RGCN components by propagating explanations through intermediate components to subsequent layers. For example, given an intermediate tensor $t$ generated by the RGCN, we compute explanations in two steps. First, we compute the importance of each node to generate $t$. Second, we compute the importance of $t$ to generate the final prediction. Finally, each node's final explanation is the sum of the nodes contribution to $t$ combined with $t$'s contribution to the final prediction. More formally, for a model defined as in Eq.~\ref{eq:tkgr}, the derivation of the final predicted score for any node $n$ can be expressed as:
\begin{equation}
\label{eq:final_eq}
\begin{aligned}
    I_n &= \frac{\partial s}{\partial n}\\
    &= \frac{\partial s}{\partial f}\frac{\partial f}{\partial RGCN(\mathcal{G}_t)}\\
    &=f'(\mathcal{H}_t)\frac{\partial f}{\partial RGCN(\mathcal{G}_t)}\\
    &=IG(\mathcal{H}_t)\cdot I_n^{RGCN}\\
    \end{aligned}
\end{equation}
\section{Experiments}
\begin{table*}[]
\begin{tabularx}{\textwidth}{l >{\raggedright\arraybackslash}X >{\raggedright\arraybackslash}X >{\raggedright\arraybackslash}X>
{\raggedright\arraybackslash}X >{\raggedright\arraybackslash}X >{\raggedright\arraybackslash}X>
{\raggedright\arraybackslash}X >{\raggedright\arraybackslash}X >{\raggedright\arraybackslash}X}
\toprule
\multirow{3}{*}{Models} & \multicolumn{9}{c}{N=5}  \\ \cmidrule{2-10} 
& \multicolumn{3}{c}{ICEWS14} & \multicolumn{3}{c}{ICEWS0515} & \multicolumn{3}{c}{ICEWS18} \\ \cmidrule{2-10} 
& Fidelity     & Stability  &Time Cost & Fidelity      & Stability   &Time Cost & Fidelity     & Stability  &Time Cost      \\ \midrule
TLogic      & 0.50         & 0.67         & 1.0          & 0.58          & 0.55         & 1.1         & 0.55         & 0.58         & 1.3              \\
TimeTraveler    & 0.59   & 0.69   & 1.6    & 0.57          & 0.54        & 1.5        & 0.55         & 0.68         & 1.7              \\
per-RE-NET    & 0.77   & 0.85  & 76.6          & 0.76          & 0.87         & 70.9         & 0.86        & 0.82         & 78.0             \\
per-CEM          & 0.81         & 0.95           & 97.8          & 0.88         & 0.94         & 92.2        & 0.83        & 0.89  & 98.5              \\
IG-RE-NET     & 0.48    & 0.55    & 2.5   & 0.47          & 0.57         & 3.2         & 0.41         & 0.45         & 2.4         \\
IG-CEM   & 0.57   & 0.53   & 1.3   & 0.44   & 0.46   & 1.5 & 0.50         & 0.45         & 1.9 \\ \midrule
xRE-NET     & 0.61    & 0.79    & 1.2   & 0.73          & 0.77         & 1.1         & 0.65         & 0.73         & 1.3          \\
xCEM   & 0.70   & 0.82   & 1.0   & 0.71   & 0.78   & 1.0 & 0.74         & 0.76         & 1.0 \\ \bottomrule
\end{tabularx}
\caption{Automatic evaluation scores of each explainable TKGR model under three datasets. $N$ denotes the number of nodes output by the models as an explanation.}
\label{tab:n5}
\end{table*}

In this section, we conduct extensive experiments to evaluate our proposed GradXKG's performance on RGCN-based TKGR models. Specifically, we investigate four research questions:
\begin{itemize}
    \item \emph{Does GradXKG provide more sufficient explanations than model-specific approaches for all evaluated RGCN-based TKGR models?}
    \item \emph{Compared to perturbation-based methods, does GradXKG offer comparable or superior explanation quality with lower time costs?}
    \item \emph{Does the proposed RGCN explainer in GradXKG contribute to explanation quality?}
    \item \emph{Can GradXKG provide relevant, validated explanations?}
\end{itemize}
\subsection{Experimental Settings}
\subsubsection{Datasets} Currently, there are no TKGR datasets designed specifically for explainable TKGR model evaluation that contain explanations associated with each quadruple (fact). Therefore, we first utilize the standard ICEWS14~\cite{garcia2018learning}, ICEWS18~\cite{jin2020recurrent}, and ICEWS0515~\cite{garcia2018learning} datasets from the Integrated Crisis Early Warning System to evaluate our proposed GradXKG method implemented on TKGR models, following previous work. We then designed an automatic evaluation approach that is based on faithfulness and trustworthiness criteria and does not require a golden standard testing set with explanations. Additionally, we conducted a human evaluation to further assess the performance of the generated explanations.
\subsubsection{Baselines} We implemented our approach on two state-of-the-art RGCN-based TKGR models: CEM~\cite{li2022complex} and RE-NET~\cite{jin2020recurrent}. Since there are currently no universal or gradient-based explanation methods for TKGR models to compare against, we implemented two model-specific explainable TKGR methods as baselines: TimeTraveler~\cite{sun2021timetraveler} and TLogic~\cite{liu2022tlogic}. We also implemented a simple perturbation-based explanation method on both CEM and RE-NET as additional baselines to benchmark GradXKG. To analyze the effectiveness of the proposed RGCN explainer, we conducted an ablation study by replacing RGCN explainer with integrated gradients. In total, we compared eight methods:
\textbf{xCEM:} A curriculum learning-based RGCN TKGR model. xCEM denotes that GradXKG is implemented on CEM. \textbf{xRE-NET:} A autoregressive RGCN TKGR model. xRE-NET means GradXKG implemented version. \textbf{per-CEM:} CEM model with perturbation-based method-generated explanations. \textbf{per-RE-NET:} RE-NET model with perturbation-based method-generated explanations \textbf{IG-CEM:} The whole CEM model is explained by integrated gradients. \textbf{IG-RE-NET:} The whole RE-NET model is explained by integrated gradients. \textbf{TimeTraveler:} A RL-based explainable TKGR model. This model does not contain RGCN architecture. \textbf{TLogic:} A temporal random walk-based explainable TKGR model that does not use RGCN.
\begin{table*}[]
\begin{tabularx}{\textwidth}{l >{\raggedright\arraybackslash}X >{\raggedright\arraybackslash}X >{\raggedright\arraybackslash}X>
{\raggedright\arraybackslash}X >{\raggedright\arraybackslash}X >{\raggedright\arraybackslash}X>
{\raggedright\arraybackslash}X >{\raggedright\arraybackslash}X >{\raggedright\arraybackslash}X}
\toprule
\multirow{3}{*}{Models} & \multicolumn{9}{c}{N=9}  \\ \cmidrule{2-10} 
& \multicolumn{3}{c}{ICEWS14} & \multicolumn{3}{c}{ICEWS0515} & \multicolumn{3}{c}{ICEWS18} \\ \cmidrule{2-10} 
& Fidelity     & Stability  &Time Cost & Fidelity      & Stability   &Time Cost & Fidelity     & Stability  &Time Cost      \\ \midrule
TLogic      & 0.44         & 0.53         & 1.3          & 0.47          & 0.51         & 1.7         & 0.40         & 0.54         & 2.1              \\
TimeTraveler    & 0.50   & 0.49   & 3.0    & 0.51          & 0.52        & 2.7        & 0.40         & 0.56         & 3.2              \\
per-RE-NET    & 0.78   & 0.87  & 76.5          & 0.88          & 0.83         & 70.9         & 0.71         & 0.90         & 78.1              \\
per-CEM          & 0.83         & 0.93           & 97.6          & 0.75         & 0.95         & 92.3         & 0.80        & 0.94  & 98.5              \\
IG-RE-NET     & 0.41    & 0.53    & 2.4   & 0.43          & 0.55         & 3.1         & 0.44         & 0.54         & 2.5         \\
IG-CEM   & 0.47   & 0.52   & 1.5   & 0.46   & 0.40   & 1.3 & 0.36         & 0.42         & 1.9 \\ \midrule
xRE-NET     & 0.63    & 0.77    & 1.4   & 0.71          & 0.75         & 1.3         & 0.70         & 0.78         & 1.8          \\
xCEM   & 0.74   & 0.85   & 1.0   & 0.70   & 0.78   & 1.0 & 0.72         & 0.74         & 1.0 \\ \bottomrule
\end{tabularx}
\caption{Fidelity and stability scores of each explainable TKGR model under three datasets. $N$ denotes the number of nodes output by the models as an explanation.}
\label{tab:n9}
\end{table*}
\subsubsection{Evaluation Metrics}
Given the lack of explainable TKGR model evaluation datasets, we assessed the quality of generated explanations through both automatic and human evaluation. We utilized three automatic evaluation metrics, namely fidelity, stability, and time complexity. We also used three human evaluation criteria, namely validity, relevance, and sufficiency. 

Following previous work on GNN explanation~\cite{yuan2022explainability}, we considered the following two automatic criteria. 1) Fidelity, which measures whether the explanations are faithfully important to the model’s predictions. Formally, fidelity can be defined by measuring the prediction gap between the original graph input and graph input that lacks important nodes: 
\begin{equation}
Fidelity = \frac{1}{N} \sum_{i=1}^{N}(f(\mathcal{G}_i)-f(\mathcal{G}_i^{m_i}))
\end{equation}where $f(\cdot)$ denotes the TKGR model, $\mathcal{G}_i$ is the input graph, and $\mathcal{G}_i^{m_i}$ is the graph with important node $m_i$ removed. 2) Stability, which measures the ability of the explainable TKGR model to produce a consistent explanation when the input graph is slightly altered or perturbed. We measure stability as follows:
\begin{equation}
Stability=\frac{|\mathcal{N}_p\cap \mathcal{N}_o|}{|\mathcal{N}_o|}
\end{equation}where $\mathcal{N}_p$ is the set of important nodes generated when the input graph is altered and $\mathcal{N}_o$ is the set of important nodes generated when the input graph is the original one. In addition to fidelity and stability, we also report the time complexity of each explanation method.

In addition to the automatic evaluation, we also included three human evaluation metrics as follows. Two annotators evaluated explanations based on three metrics: validity, relevance, and sufficiency criteria using a 3-point scale (1=low, 3=high). Validity measures whether the explanation is logically valid based on the annotator's knowledge, indicating if the model learned the expected features. Relevance measures whether selected nodes are connected meaningfully to the target node. Sufficiency measures if the explanation contains enough nodes to sufficiently understand the reasoning behind the prediction, since explanations with too few nodes may lack context. Note that there may be a trade-off between relevance and sufficiency, since more nodes may increase sufficiency but decrease relevance. The detailed evaluation guideline is provided in Appendix~\ref{sec:anno}.
\subsection{Automatic Evaluation Results}
We conduct two experiments where the number of nodes output in the explanation is constrained to 5 and 9, representing reasoning paths of lengths 3 and 4, respectively. This allows us to evaluate the performance of GradXKG compared to TimeTravler and TLogic, which have constraints on the length of the reasoning path they can output. The results in Table~\ref{tab:n5} demonstrate that with 5 node explanations, GradXKG achieves higher fidelity and stability scores than TimeTravler and TLogic on all datasets. Specifically, on the ICEWS05015 dataset, the xCEM and xRE-NET variants of GradXKG obtain fidelity scores of 0.73 and 0.71, and stability scores of 0.77 and 0.78; compared to 0.58 and 0.57 fidelity, and 0.55 and 0.54 stability for TimeTravler and TLogic, respectively. This shows that despite being a universal explainer without special architectural constraints, GradXKG can generate more trustworthy explanations than previous constrained models.

We also compare GradXKG to perturbation-based methods, which intrinsically optimize explanation fidelity and stability. As expected, these methods slightly outperform GradXKG but at a drastically higher computational cost. For instance, per-CEM takes 97.8 times longer than xCEM on the ICEWS14 dataset, demonstrating the efficiency of GradXKG's ``explain-per-use'' approach.

Note that TLogic is not theoretically runnable when the reasoning path length is longer than 5. Therefore, we also set the generated number of nodes requirement to be 9 and show the results in Table.~\ref{tab:n9}. Compared to the 5-node experiment, the 9-node experiment shows that most perturbation and gradient-based methods maintain a similar performance across all three automatic criteria. This is because they score all nodes and select the top ones. However, TimeTravler and TLogic slow down considerably and have lower fidelity and stability as they are constrained by needing to traverse explanation trajectories. Overall, the experiments demonstrate GradXKG's ability to generate trustworthy explanations efficiently despite its model-agnostic nature.

\subsection{Human Evaluation and Qualitative Analysis}
\textbf{Human Evaluation} To further evaluate the quality of the generated explanations, we randomly selected 50 generated explanations from each compared method and asked two experienced annotators to rate the quality of the explanations based on the aforementioned criteria. As shown in Fig.~\ref{fig:box}, the saliency map explanations generated by GradXKG achieved higher sufficiency scores compared to the other methods. This demonstrates that compared to reasoning path explanations, like those from TLogic and TimeTraveler, the proposed GradXKG method can provide more contextual information by assigning importance scores to all nodes in the TKG. Furthermore, the relevance scores of xCEM are comparable to those of TLogic, indicating that the importance scores of each node are aligned with its contribution to the prediction. We also observed differences in the evaluation results between xRE-NET and xCEM even when using the same GradXKG method. This discrepancy can be attributed to the fact that the explainer's performance also heavily depends on the capabilities of the original model.

\textbf{Qualitative Analysis} We include two saliency maps generated by xCEM, the best variant of GradXKG, when the following query is submitted to the model: \emph{(National Council for Peace and Order of Thailand, Make Statement, ?, 2014-07-24)} and \emph{(John Kerry, Engage in negotiation, ?, 2014-07-24)}. The correct answers to these two queries are \emph{Thailand} and \emph{Benjamin Netanyahu}, respectively. As shown in Fig.~\ref{fig:example}, GradXKG revealed that the CEM model relied on nodes/entities that are highly related to the queries in order to make the correct predictions. For instance, for the query \emph{(John Kerry, Engage in negotiation, ?, 2014-07-24)}, CEM attended to other U.S. government officials like Barack Obama and international political organizations such as ``UN Security Council'' as well as politicians in the Middle East region. This demonstrates that with the proposed GradXKG method, RGCN-based TKGR models' predictions can be accurately interpreted and explained. 
\begin{figure*}[]
\centering
\includegraphics[width=0.95\textwidth]{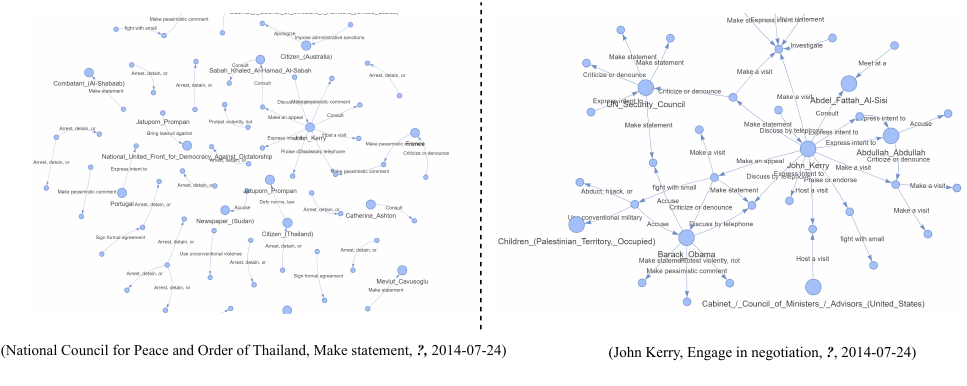}
\caption{The saliency maps generated by xCEM, where big nodes are the selected important nodes by GradXKG. We omit the rest of the TKGs and timestamps for better illustration.}
\label{fig:example}
\end{figure*}

\begin{figure*}[]
\centering
\includegraphics[width=0.95\textwidth]{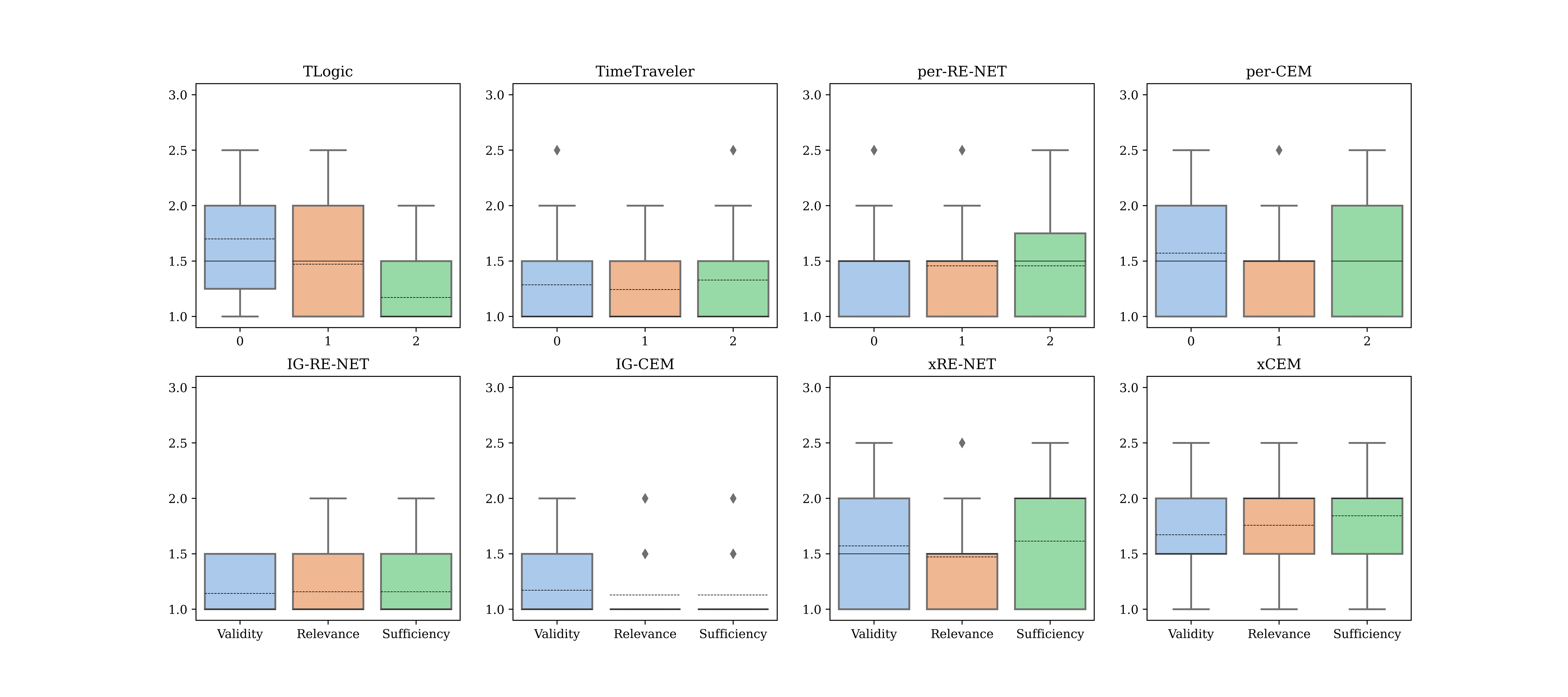}
\caption{The box plots of human evaluation for each criterion of each baseline method. The dashed line denotes the mean value and the bold line indicates the median value.}
\label{fig:box}
\end{figure*}
\subsection{Ablation Study}
To evaluate the effectiveness of the proposed RGCN explainer, we also conducted an ablation study by replacing the RGCN explainer with integrated gradients (IG). The node importance score is calculated by averaging the IG scores across the whole vector representation of that node. As shown in Tables~\ref{tab:n5} and~\ref{tab:n9}, the IG-based RE-NET and CEM methods perform much worse in terms of fidelity and stability. This demonstrates the effectiveness and trustworthiness of the proposed RGCN explainer. Furthermore, as illustrated in Fig.~\ref{fig:box}, the human evaluation also suggests that IG-only RE-NET or CEM methods cannot provide a saliency map with high validity, relevance, and sufficiency. Together, this shows the necessity and effectiveness of the proposed RGCN explainer.

\section{Conclusion}
In this paper, we propose GradXKG, a universal explain-per-use temporal knowledge graph explainer. The proposed approach can be applied to most RGCN-based TKGR models to generate saliency maps that indicate node importance (contribution) towards the prediction as the explanation. Our proposed explainer model tracks the gradient flow in the whole model by introducing, and combining, two explainers: a Grad-CAM-based RGCN explainer and an integrated gradient-based top-layer explainer. Our extensive experiments demonstrate that with GradXKG, most state-of-the-art TKGR models can be explained with high fidelity and stability in a time-efficient manner. 

\bibliographystyle{ACM-Reference-Format}
\bibliography{reference}


\begin{thebibliography}{34}


\ifx \showCODEN    \undefined \def \showCODEN     #1{\unskip}     \fi
\ifx \showDOI      \undefined \def \showDOI       #1{#1}\fi
\ifx \showISBNx    \undefined \def \showISBNx     #1{\unskip}     \fi
\ifx \showISBNxiii \undefined \def \showISBNxiii  #1{\unskip}     \fi
\ifx \showISSN     \undefined \def \showISSN      #1{\unskip}     \fi
\ifx \showLCCN     \undefined \def \showLCCN      #1{\unskip}     \fi
\ifx \shownote     \undefined \def \shownote      #1{#1}          \fi
\ifx \showarticletitle \undefined \def \showarticletitle #1{#1}   \fi
\ifx \showURL      \undefined \def \showURL       {\relax}        \fi
\providecommand\bibfield[2]{#2}
\providecommand\bibinfo[2]{#2}
\providecommand\natexlab[1]{#1}
\providecommand\showeprint[2][]{arXiv:#2}

\bibitem[Bach et~al\mbox{.}(2015)]%
        {bach2015pixel}
\bibfield{author}{\bibinfo{person}{Sebastian Bach}, \bibinfo{person}{Alexander
  Binder}, \bibinfo{person}{Gr{\'e}goire Montavon}, \bibinfo{person}{Frederick
  Klauschen}, \bibinfo{person}{Klaus-Robert M{\"u}ller}, {and}
  \bibinfo{person}{Wojciech Samek}.} \bibinfo{year}{2015}\natexlab{}.
\newblock \showarticletitle{On pixel-wise explanations for non-linear
  classifier decisions by layer-wise relevance propagation}.
\newblock \bibinfo{journal}{\emph{PloS one}} \bibinfo{volume}{10},
  \bibinfo{number}{7} (\bibinfo{year}{2015}), \bibinfo{pages}{e0130140}.
\newblock


\bibitem[Deng et~al\mbox{.}(2020)]%
        {deng2020dynamic}
\bibfield{author}{\bibinfo{person}{Songgaojun Deng}, \bibinfo{person}{Huzefa
  Rangwala}, {and} \bibinfo{person}{Yue Ning}.}
  \bibinfo{year}{2020}\natexlab{}.
\newblock \showarticletitle{Dynamic knowledge graph based multi-event
  forecasting}. In \bibinfo{booktitle}{\emph{Proceedings of the 26th ACM SIGKDD
  International Conference on Knowledge Discovery \& Data Mining}}.
  \bibinfo{pages}{1585--1595}.
\newblock


\bibitem[Deng et~al\mbox{.}(2019)]%
        {deng2019knowledge}
\bibfield{author}{\bibinfo{person}{Shumin Deng}, \bibinfo{person}{Ningyu
  Zhang}, \bibinfo{person}{Wen Zhang}, \bibinfo{person}{Jiaoyan Chen},
  \bibinfo{person}{Jeff~Z Pan}, {and} \bibinfo{person}{Huajun Chen}.}
  \bibinfo{year}{2019}\natexlab{}.
\newblock \showarticletitle{Knowledge-driven stock trend prediction and
  explanation via temporal convolutional network}. In
  \bibinfo{booktitle}{\emph{Companion Proceedings of The 2019 World Wide Web
  Conference}}. \bibinfo{pages}{678--685}.
\newblock


\bibitem[Diao et~al\mbox{.}(2021)]%
        {diao2021research}
\bibfield{author}{\bibinfo{person}{Lijuan Diao}, \bibinfo{person}{Wei Yang},
  \bibinfo{person}{Penghua Zhu}, \bibinfo{person}{Gaofang Cao},
  \bibinfo{person}{Shoujun Song}, {and} \bibinfo{person}{Yang Kong}.}
  \bibinfo{year}{2021}\natexlab{}.
\newblock \showarticletitle{The research of clinical temporal knowledge graph
  based on deep learning}.
\newblock \bibinfo{journal}{\emph{Journal of Intelligent \& Fuzzy Systems}}
  \bibinfo{volume}{41}, \bibinfo{number}{3} (\bibinfo{year}{2021}),
  \bibinfo{pages}{4265--4274}.
\newblock


\bibitem[Garcia-Duran et~al\mbox{.}(2018)]%
        {garcia2018learning}
\bibfield{author}{\bibinfo{person}{Alberto Garcia-Duran},
  \bibinfo{person}{Sebastijan Duman{\v{c}}i{\'c}}, {and}
  \bibinfo{person}{Mathias Niepert}.} \bibinfo{year}{2018}\natexlab{}.
\newblock \showarticletitle{Learning Sequence Encoders for Temporal Knowledge
  Graph Completion}. In \bibinfo{booktitle}{\emph{Proceedings of the 2018
  Conference on Empirical Methods in Natural Language Processing}}.
  \bibinfo{pages}{4816--4821}.
\newblock


\bibitem[Gottschalk and Demidova(2018)]%
        {gottschalk2018eventkg}
\bibfield{author}{\bibinfo{person}{Simon Gottschalk} {and}
  \bibinfo{person}{Elena Demidova}.} \bibinfo{year}{2018}\natexlab{}.
\newblock \showarticletitle{Eventkg: A multilingual event-centric temporal
  knowledge graph}. In \bibinfo{booktitle}{\emph{The Semantic Web: 15th
  International Conference, ESWC 2018, Heraklion, Crete, Greece, June 3--7,
  2018, Proceedings 15}}. Springer, \bibinfo{pages}{272--287}.
\newblock


\bibitem[Han et~al\mbox{.}(2021a)]%
        {han2021explainable}
\bibfield{author}{\bibinfo{person}{Zhen Han}, \bibinfo{person}{Peng Chen},
  \bibinfo{person}{Yunpu Ma}, {and} \bibinfo{person}{Volker Tresp}.}
  \bibinfo{year}{2021}\natexlab{a}.
\newblock \showarticletitle{Explainable Subgraph Reasoning for Forecasting on
  Temporal Knowledge Graphs}. In \bibinfo{booktitle}{\emph{International
  Conference on Learning Representations}}.
\newblock
\urldef\tempurl%
\url{https://openreview.net/forum?id=pGIHq1m7PU}
\showURL{%
\tempurl}


\bibitem[Han et~al\mbox{.}(2021b)]%
        {han2021learning}
\bibfield{author}{\bibinfo{person}{Zhen Han}, \bibinfo{person}{Zifeng Ding},
  \bibinfo{person}{Yunpu Ma}, \bibinfo{person}{Yujia Gu}, {and}
  \bibinfo{person}{Volker Tresp}.} \bibinfo{year}{2021}\natexlab{b}.
\newblock \showarticletitle{Learning neural ordinary equations for forecasting
  future links on temporal knowledge graphs}. In
  \bibinfo{booktitle}{\emph{Proceedings of the 2021 conference on empirical
  methods in natural language processing}}. \bibinfo{pages}{8352--8364}.
\newblock


\bibitem[Huang et~al\mbox{.}(2022)]%
        {huang2022graphlime}
\bibfield{author}{\bibinfo{person}{Qiang Huang}, \bibinfo{person}{Makoto
  Yamada}, \bibinfo{person}{Yuan Tian}, \bibinfo{person}{Dinesh Singh}, {and}
  \bibinfo{person}{Yi Chang}.} \bibinfo{year}{2022}\natexlab{}.
\newblock \showarticletitle{Graphlime: Local interpretable model explanations
  for graph neural networks}.
\newblock \bibinfo{journal}{\emph{IEEE Transactions on Knowledge and Data
  Engineering}} (\bibinfo{year}{2022}).
\newblock


\bibitem[Jia et~al\mbox{.}(2021)]%
        {jia2021complex}
\bibfield{author}{\bibinfo{person}{Zhen Jia}, \bibinfo{person}{Soumajit
  Pramanik}, \bibinfo{person}{Rishiraj Saha~Roy}, {and}
  \bibinfo{person}{Gerhard Weikum}.} \bibinfo{year}{2021}\natexlab{}.
\newblock \showarticletitle{Complex temporal question answering on knowledge
  graphs}. In \bibinfo{booktitle}{\emph{Proceedings of the 30th ACM
  international conference on information \& knowledge management}}.
  \bibinfo{pages}{792--802}.
\newblock


\bibitem[Jin et~al\mbox{.}(2020)]%
        {jin2020recurrent}
\bibfield{author}{\bibinfo{person}{Woojeong Jin}, \bibinfo{person}{Meng Qu},
  \bibinfo{person}{Xisen Jin}, {and} \bibinfo{person}{Xiang Ren}.}
  \bibinfo{year}{2020}\natexlab{}.
\newblock \showarticletitle{Recurrent Event Network: Autoregressive Structure
  Inferenceover Temporal Knowledge Graphs}. In
  \bibinfo{booktitle}{\emph{Proceedings of the 2020 Conference on Empirical
  Methods in Natural Language Processing (EMNLP)}}.
  \bibinfo{pages}{6669--6683}.
\newblock


\bibitem[Jung et~al\mbox{.}(2021)]%
        {jung2021learning}
\bibfield{author}{\bibinfo{person}{Jaehun Jung}, \bibinfo{person}{Jinhong
  Jung}, {and} \bibinfo{person}{U Kang}.} \bibinfo{year}{2021}\natexlab{}.
\newblock \showarticletitle{Learning to walk across time for interpretable
  temporal knowledge graph completion}. In
  \bibinfo{booktitle}{\emph{Proceedings of the 27th ACM SIGKDD Conference on
  Knowledge Discovery \& Data Mining}}. \bibinfo{pages}{786--795}.
\newblock


\bibitem[Li et~al\mbox{.}(2022)]%
        {li2022complex}
\bibfield{author}{\bibinfo{person}{Zixuan Li}, \bibinfo{person}{Saiping Guan},
  \bibinfo{person}{Xiaolong Jin}, \bibinfo{person}{Weihua Peng},
  \bibinfo{person}{Yajuan Lyu}, \bibinfo{person}{Yong Zhu},
  \bibinfo{person}{Long Bai}, \bibinfo{person}{Wei Li},
  \bibinfo{person}{Jiafeng Guo}, {and} \bibinfo{person}{Xueqi Cheng}.}
  \bibinfo{year}{2022}\natexlab{}.
\newblock \showarticletitle{Complex Evolutional Pattern Learning for Temporal
  Knowledge Graph Reasoning}. In \bibinfo{booktitle}{\emph{Proceedings of the
  60th Annual Meeting of the Association for Computational Linguistics (Volume
  2: Short Papers)}}. \bibinfo{pages}{290--296}.
\newblock


\bibitem[Li et~al\mbox{.}(2021a)]%
        {li2021search}
\bibfield{author}{\bibinfo{person}{Zixuan Li}, \bibinfo{person}{Xiaolong Jin},
  \bibinfo{person}{Saiping Guan}, \bibinfo{person}{Wei Li},
  \bibinfo{person}{Jiafeng Guo}, \bibinfo{person}{Yuanzhuo Wang}, {and}
  \bibinfo{person}{Xueqi Cheng}.} \bibinfo{year}{2021}\natexlab{a}.
\newblock \showarticletitle{Search from History and Reason for Future:
  Two-stage Reasoning on Temporal Knowledge Graphs}. In
  \bibinfo{booktitle}{\emph{Proceedings of the 59th Annual Meeting of the
  Association for Computational Linguistics and the 11th International Joint
  Conference on Natural Language Processing (Volume 1: Long Papers)}}.
  \bibinfo{pages}{4732--4743}.
\newblock


\bibitem[Li et~al\mbox{.}(2021b)]%
        {li2021temporal}
\bibfield{author}{\bibinfo{person}{Zixuan Li}, \bibinfo{person}{Xiaolong Jin},
  \bibinfo{person}{Wei Li}, \bibinfo{person}{Saiping Guan},
  \bibinfo{person}{Jiafeng Guo}, \bibinfo{person}{Huawei Shen},
  \bibinfo{person}{Yuanzhuo Wang}, {and} \bibinfo{person}{Xueqi Cheng}.}
  \bibinfo{year}{2021}\natexlab{b}.
\newblock \showarticletitle{Temporal knowledge graph reasoning based on
  evolutional representation learning}. In
  \bibinfo{booktitle}{\emph{Proceedings of the 44th international ACM SIGIR
  conference on research and development in information retrieval}}.
  \bibinfo{pages}{408--417}.
\newblock


\bibitem[Lin et~al\mbox{.}(2023)]%
        {lin2023techs}
\bibfield{author}{\bibinfo{person}{Qika Lin}, \bibinfo{person}{Jun Liu},
  \bibinfo{person}{Rui Mao}, \bibinfo{person}{Fangzhi Xu}, {and}
  \bibinfo{person}{Erik Cambria}.} \bibinfo{year}{2023}\natexlab{}.
\newblock \showarticletitle{TECHS: Temporal Logical Graph Networks for
  Explainable Extrapolation Reasoning}. In
  \bibinfo{booktitle}{\emph{Proceedings of the 61st Annual Meeting of the
  Association for Computational Linguistics (Volume 1: Long Papers)}}.
  \bibinfo{pages}{1281--1293}.
\newblock


\bibitem[Liu et~al\mbox{.}(2022)]%
        {liu2022tlogic}
\bibfield{author}{\bibinfo{person}{Yushan Liu}, \bibinfo{person}{Yunpu Ma},
  \bibinfo{person}{Marcel Hildebrandt}, \bibinfo{person}{Mitchell Joblin},
  {and} \bibinfo{person}{Volker Tresp}.} \bibinfo{year}{2022}\natexlab{}.
\newblock \showarticletitle{Tlogic: Temporal logical rules for explainable link
  forecasting on temporal knowledge graphs}. In
  \bibinfo{booktitle}{\emph{Proceedings of the AAAI conference on artificial
  intelligence}}, Vol.~\bibinfo{volume}{36}. \bibinfo{pages}{4120--4127}.
\newblock


\bibitem[Oquab et~al\mbox{.}(2015)]%
        {oquab2015object}
\bibfield{author}{\bibinfo{person}{Maxime Oquab}, \bibinfo{person}{L{\'e}on
  Bottou}, \bibinfo{person}{Ivan Laptev}, {and} \bibinfo{person}{Josef Sivic}.}
  \bibinfo{year}{2015}\natexlab{}.
\newblock \showarticletitle{Is object localization for free?-weakly-supervised
  learning with convolutional neural networks}. In
  \bibinfo{booktitle}{\emph{Proceedings of the IEEE conference on computer
  vision and pattern recognition}}. \bibinfo{pages}{685--694}.
\newblock


\bibitem[Pope et~al\mbox{.}(2019)]%
        {pope2019explainability}
\bibfield{author}{\bibinfo{person}{Phillip~E Pope}, \bibinfo{person}{Soheil
  Kolouri}, \bibinfo{person}{Mohammad Rostami}, \bibinfo{person}{Charles~E
  Martin}, {and} \bibinfo{person}{Heiko Hoffmann}.}
  \bibinfo{year}{2019}\natexlab{}.
\newblock \showarticletitle{Explainability methods for graph convolutional
  neural networks}. In \bibinfo{booktitle}{\emph{Proceedings of the IEEE/CVF
  conference on computer vision and pattern recognition}}.
  \bibinfo{pages}{10772--10781}.
\newblock


\bibitem[Ribeiro et~al\mbox{.}(2016)]%
        {ribeiro2016should}
\bibfield{author}{\bibinfo{person}{Marco~Tulio Ribeiro},
  \bibinfo{person}{Sameer Singh}, {and} \bibinfo{person}{Carlos Guestrin}.}
  \bibinfo{year}{2016}\natexlab{}.
\newblock \showarticletitle{" Why should i trust you?" Explaining the
  predictions of any classifier}. In \bibinfo{booktitle}{\emph{Proceedings of
  the 22nd ACM SIGKDD international conference on knowledge discovery and data
  mining}}. \bibinfo{pages}{1135--1144}.
\newblock


\bibitem[Saxena et~al\mbox{.}(2021)]%
        {saxena2021question}
\bibfield{author}{\bibinfo{person}{A Saxena}, \bibinfo{person}{S Chakrabarti},
  {and} \bibinfo{person}{P Talukdar}.} \bibinfo{year}{2021}\natexlab{}.
\newblock \showarticletitle{Question answering over temporal knowledge graphs}.
  In \bibinfo{booktitle}{\emph{ACL-IJCNLP 2021-59th Annual Meeting of the
  Association for Computational Linguistics and the 11th International Joint
  Conference on Natural Language Processing, Proceedings of the Conference}}.
  Association for Computational Linguistics (ACL), \bibinfo{pages}{6663--6676}.
\newblock


\bibitem[Schlichtkrull et~al\mbox{.}(2018)]%
        {schlichtkrull2018modeling}
\bibfield{author}{\bibinfo{person}{Michael Schlichtkrull},
  \bibinfo{person}{Thomas~N Kipf}, \bibinfo{person}{Peter Bloem},
  \bibinfo{person}{Rianne Van Den~Berg}, \bibinfo{person}{Ivan Titov}, {and}
  \bibinfo{person}{Max Welling}.} \bibinfo{year}{2018}\natexlab{}.
\newblock \showarticletitle{Modeling relational data with graph convolutional
  networks}. In \bibinfo{booktitle}{\emph{The Semantic Web: 15th International
  Conference, ESWC 2018, Heraklion, Crete, Greece, June 3--7, 2018, Proceedings
  15}}. Springer, \bibinfo{pages}{593--607}.
\newblock


\bibitem[Schnake et~al\mbox{.}(2021)]%
        {schnake2021higher}
\bibfield{author}{\bibinfo{person}{Thomas Schnake}, \bibinfo{person}{Oliver
  Eberle}, \bibinfo{person}{Jonas Lederer}, \bibinfo{person}{Shinichi
  Nakajima}, \bibinfo{person}{Kristof~T Sch{\"u}tt},
  \bibinfo{person}{Klaus-Robert M{\"u}ller}, {and}
  \bibinfo{person}{Gr{\'e}goire Montavon}.} \bibinfo{year}{2021}\natexlab{}.
\newblock \showarticletitle{Higher-order explanations of graph neural networks
  via relevant walks}.
\newblock \bibinfo{journal}{\emph{IEEE transactions on pattern analysis and
  machine intelligence}} \bibinfo{volume}{44}, \bibinfo{number}{11}
  (\bibinfo{year}{2021}), \bibinfo{pages}{7581--7596}.
\newblock


\bibitem[Selvaraju et~al\mbox{.}(2017)]%
        {selvaraju2017grad}
\bibfield{author}{\bibinfo{person}{Ramprasaath~R Selvaraju},
  \bibinfo{person}{Michael Cogswell}, \bibinfo{person}{Abhishek Das},
  \bibinfo{person}{Ramakrishna Vedantam}, \bibinfo{person}{Devi Parikh}, {and}
  \bibinfo{person}{Dhruv Batra}.} \bibinfo{year}{2017}\natexlab{}.
\newblock \showarticletitle{Grad-cam: Visual explanations from deep networks
  via gradient-based localization}. In \bibinfo{booktitle}{\emph{Proceedings of
  the IEEE international conference on computer vision}}.
  \bibinfo{pages}{618--626}.
\newblock


\bibitem[Sun et~al\mbox{.}(2021)]%
        {sun2021timetraveler}
\bibfield{author}{\bibinfo{person}{Haohai Sun}, \bibinfo{person}{Jialun Zhong},
  \bibinfo{person}{Yunpu Ma}, \bibinfo{person}{Zhen Han}, {and}
  \bibinfo{person}{Kun He}.} \bibinfo{year}{2021}\natexlab{}.
\newblock \showarticletitle{TimeTraveler: Reinforcement Learning for Temporal
  Knowledge Graph Forecasting}. In \bibinfo{booktitle}{\emph{Proceedings of the
  2021 Conference on Empirical Methods in Natural Language Processing}}.
  \bibinfo{pages}{8306--8319}.
\newblock


\bibitem[Sundararajan et~al\mbox{.}(2017)]%
        {sundararajan2017axiomatic}
\bibfield{author}{\bibinfo{person}{Mukund Sundararajan}, \bibinfo{person}{Ankur
  Taly}, {and} \bibinfo{person}{Qiqi Yan}.} \bibinfo{year}{2017}\natexlab{}.
\newblock \showarticletitle{Axiomatic attribution for deep networks}. In
  \bibinfo{booktitle}{\emph{International conference on machine learning}}.
  PMLR, \bibinfo{pages}{3319--3328}.
\newblock


\bibitem[Trivedi et~al\mbox{.}(2017)]%
        {trivedi2017know}
\bibfield{author}{\bibinfo{person}{Rakshit Trivedi}, \bibinfo{person}{Hanjun
  Dai}, \bibinfo{person}{Yichen Wang}, {and} \bibinfo{person}{Le Song}.}
  \bibinfo{year}{2017}\natexlab{}.
\newblock \showarticletitle{Know-evolve: Deep temporal reasoning for dynamic
  knowledge graphs}. In \bibinfo{booktitle}{\emph{international conference on
  machine learning}}. PMLR, \bibinfo{pages}{3462--3471}.
\newblock


\bibitem[Vu and Thai(2020)]%
        {vu2020pgm}
\bibfield{author}{\bibinfo{person}{Minh Vu} {and} \bibinfo{person}{My~T Thai}.}
  \bibinfo{year}{2020}\natexlab{}.
\newblock \showarticletitle{Pgm-explainer: Probabilistic graphical model
  explanations for graph neural networks}.
\newblock \bibinfo{journal}{\emph{Advances in neural information processing
  systems}}  \bibinfo{volume}{33} (\bibinfo{year}{2020}),
  \bibinfo{pages}{12225--12235}.
\newblock


\bibitem[Xia et~al\mbox{.}(2022)]%
        {xia2022metatkg}
\bibfield{author}{\bibinfo{person}{Yuwei Xia}, \bibinfo{person}{Mengqi Zhang},
  \bibinfo{person}{Qiang Liu}, \bibinfo{person}{Shu Wu}, {and}
  \bibinfo{person}{Xiao-Yu Zhang}.} \bibinfo{year}{2022}\natexlab{}.
\newblock \showarticletitle{MetaTKG: Learning Evolutionary Meta-Knowledge for
  Temporal Knowledge Graph Reasoning}. In \bibinfo{booktitle}{\emph{Proceedings
  of the 2022 Conference on Empirical Methods in Natural Language Processing}}.
  \bibinfo{pages}{7230--7240}.
\newblock


\bibitem[Xu et~al\mbox{.}(2023)]%
        {xu2023temporal}
\bibfield{author}{\bibinfo{person}{Yi Xu}, \bibinfo{person}{Junjie Ou},
  \bibinfo{person}{Hui Xu}, {and} \bibinfo{person}{Luoyi Fu}.}
  \bibinfo{year}{2023}\natexlab{}.
\newblock \showarticletitle{Temporal knowledge graph reasoning with historical
  contrastive learning}. In \bibinfo{booktitle}{\emph{Proceedings of the AAAI
  Conference on Artificial Intelligence}}, Vol.~\bibinfo{volume}{37}.
  \bibinfo{pages}{4765--4773}.
\newblock


\bibitem[Ying et~al\mbox{.}(2019)]%
        {ying2019gnnexplainer}
\bibfield{author}{\bibinfo{person}{Zhitao Ying}, \bibinfo{person}{Dylan
  Bourgeois}, \bibinfo{person}{Jiaxuan You}, \bibinfo{person}{Marinka Zitnik},
  {and} \bibinfo{person}{Jure Leskovec}.} \bibinfo{year}{2019}\natexlab{}.
\newblock \showarticletitle{Gnnexplainer: Generating explanations for graph
  neural networks}.
\newblock \bibinfo{journal}{\emph{Advances in neural information processing
  systems}}  \bibinfo{volume}{32} (\bibinfo{year}{2019}).
\newblock


\bibitem[Yuan et~al\mbox{.}(2022)]%
        {yuan2022explainability}
\bibfield{author}{\bibinfo{person}{Hao Yuan}, \bibinfo{person}{Haiyang Yu},
  \bibinfo{person}{Shurui Gui}, {and} \bibinfo{person}{Shuiwang Ji}.}
  \bibinfo{year}{2022}\natexlab{}.
\newblock \showarticletitle{Explainability in graph neural networks: A
  taxonomic survey}.
\newblock \bibinfo{journal}{\emph{IEEE transactions on pattern analysis and
  machine intelligence}} \bibinfo{volume}{45}, \bibinfo{number}{5}
  (\bibinfo{year}{2022}), \bibinfo{pages}{5782--5799}.
\newblock


\bibitem[Yuan et~al\mbox{.}(2021)]%
        {yuan2021explainability}
\bibfield{author}{\bibinfo{person}{Hao Yuan}, \bibinfo{person}{Haiyang Yu},
  \bibinfo{person}{Jie Wang}, \bibinfo{person}{Kang Li}, {and}
  \bibinfo{person}{Shuiwang Ji}.} \bibinfo{year}{2021}\natexlab{}.
\newblock \showarticletitle{On explainability of graph neural networks via
  subgraph explorations}. In \bibinfo{booktitle}{\emph{International conference
  on machine learning}}. PMLR, \bibinfo{pages}{12241--12252}.
\newblock


\bibitem[Zhang et~al\mbox{.}(2020)]%
        {zhang2020relational}
\bibfield{author}{\bibinfo{person}{Zhao Zhang}, \bibinfo{person}{Fuzhen
  Zhuang}, \bibinfo{person}{Hengshu Zhu}, \bibinfo{person}{Zhiping Shi},
  \bibinfo{person}{Hui Xiong}, {and} \bibinfo{person}{Qing He}.}
  \bibinfo{year}{2020}\natexlab{}.
\newblock \showarticletitle{Relational graph neural network with hierarchical
  attention for knowledge graph completion}. In
  \bibinfo{booktitle}{\emph{Proceedings of the AAAI conference on artificial
  intelligence}}, Vol.~\bibinfo{volume}{34}. \bibinfo{pages}{9612--9619}.
\newblock


\end{thebibliography}

\appendix

\section{Annotation Guideline}
\label{sec:anno}
We provide detailed evaluation guidelines for human evaluations in this section.
\textbf{Overview}

You will be shown explanations generated by a model for temporal knowledge graph reasoning tasks. Your task is to evaluate the quality of the explanation based on 3 criteria using a 1-3 rating scale:
\begin{itemize}
    \item Validity (1=low, 3=high) - Does the explanation seem logically valid based on your own knowledge? Note that invalid explanations indicate the model may not have learned expected features.
    \item Relevance (1=low, 3=high) - Do the nodes selected seem relevant for explaining the model's prediction? Relevant nodes should meaningfully connect to the target node being explained.
    \item Sufficiency (1=low, 3=high) - Does the explanation contain enough nodes to sufficiently understand the reasoning behind the prediction? Too few nodes may lack the needed context.
\end{itemize}
The detailed criteria for each score is described below:
 Note there is a tradeoff between sufficiency and relevance - more nodes increase sufficiency but may reduce relevance.
 
 \textbf{Validity}
\begin{itemize}
    \item Assign a score of 1 if the explanation is illogical or contains factual errors based on your knowledge.
    \item Assign a score of 2 if you are uncertain of the validity or logic of the explanation.
    \item Assign a score of 3 if the explanation seems factually and logically valid.
\end{itemize}

 \textbf{Relevance}
\begin{itemize}
    \item Assign a score of 1 if the nodes seem unrelated to explaining the prediction.
    \item Assign a score of 2 if some nodes are relevant while others are not.
    \item Assign a score of 3 if all nodes clearly help explain the prediction through their connections.
\end{itemize}

 \textbf{Sufficiency}
\begin{itemize}
    \item Assign a score of 1 if the explanation does not provide enough context to understand the reasoning.
    \item Assign a score of 2 if the explanation provides some context but is still lacking.
    \item Assign a score of 3 if the explanation provides ample context to understand the reasoning.
\end{itemize}

\end{document}